\documentclass[letterpaper, 10 pt, journal, twoside]{IEEEtran}

\usepackage{geometry}
\usepackage{amsthm}
\usepackage{graphics} 
\usepackage{epsfig} 
\usepackage{times} 
\usepackage{amsmath} 
\usepackage{mathtools,amsfonts}
\usepackage{amssymb}  
\usepackage[T1, OT1]{fontenc}
\DeclareTextSymbolDefault{\dh}{T1}
\usepackage{hyperref}
\usepackage{url}

\usepackage{algorithm}
\usepackage[noend]{algpseudocode}
\usepackage{newtxtext}
\usepackage{newtxmath}
\usepackage{lettrine}
\usepackage{booktabs}
\usepackage{multirow}
\usepackage{graphicx}
\usepackage{multicol}
\usepackage{caption}
\usepackage{subcaption}
\usepackage{mwe}
\usepackage{hyperref}
\usepackage{placeins}
\hypersetup{
    colorlinks=true,
    linkcolor=blue,
    filecolor=magenta,      
    urlcolor=cyan,
}

\newtheorem{theorem}{Theorem}

\makeatletter
\let\OldStatex\Statex
\renewcommand{\Statex}[1][3]{%
  \setlength\@tempdima{\algorithmicindent}%
  \OldStatex\hskip\dimexpr#1\@tempdima\relax}
\makeatother

\def\eqnvspace{\vspace{-0.15cm}}

\begin{document}

\title{\LARGE \bf
Interleaving Graph Search and Trajectory Optimization \\ for Aggressive Quadrotor Flight
}

\author{Ramkumar Natarajan$^{1}$, Howie Choset$^{1}$ and Maxim Likhachev$^{1}$

\thanks{Manuscript received: July, 28, 2020; Revised November, 19, 2020; Accepted December, 21, 2021.}
\thanks{This paper was recommended for publication by
Editor Jonathan Roberts upon evaluation of the Associate Editor and Reviewers’ comments.} 
\thanks{The authors are with The Robotics Institute at Carnegie Mellon University, Pittsburgh, PA 15213, USA {\tt\small \{rnataraj, choset, maxim\}@cs.cmu.edu}}
\thanks{Digital Object Identifier (DOI): see top of this page.}
}

\markboth{IEEE Robotics and Automation Letters. Preprint Version. Accepted February 2021}{Natarajan \MakeLowercase{\textit{et al.}}: Interleaving Graph Search and Trajectory Optimization for Aggressive Quadrotor Flight}

\renewcommand\intercal{{\cramped{{}^\mathsf{T}}}}

\maketitle

\begin{abstract}
Quadrotors can achieve aggressive flight by tracking complex maneuvers and rapidly changing directions. Planning for aggressive flight with trajectory optimization could be incredibly fast, even in higher dimensions, and can account for dynamics of the quadrotor, however, only provides a locally optimal solution. On the other hand, planning with discrete graph search can handle non-convex spaces to guarantee optimality but suffers from exponential complexity with the dimension of search. We introduce a framework for aggressive quadrotor trajectory generation with global reasoning capabilities that combines the best of trajectory optimization and discrete graph search. Specifically, we develop a novel algorithmic framework that \textit{interleaves} these two methods to complement each other and generate trajectories with provable guarantees on completeness up to discretization. We demonstrate and quantitatively analyze the performance of our algorithm in challenging simulation environments with narrow gaps that create severe attitude constraints and push the dynamic capabilities of the quadrotor. Experiments show the benefits of the proposed algorithmic framework over standalone trajectory optimization and graph search-based planning techniques for aggressive quadrotor flight.

\end{abstract}

\section{Introduction}
\lettrine{Q}{}uadrotors' exceptional agility and ability to track and execute complex maneuvers, fly through narrow gaps and rapidly change directions make motion planning for aggressive quadrotor flight an exciting and important area of research  \cite{cutlervarpitch, mellingerprecise, andreajuggling}. In order to enable such agile capabilities, motion planning should consider the dynamics and the control limits of the robot. The three distinct approaches for motion planning with dynamics are: (a) optimal control techniques, like trajectory optimization \cite{scaramuzzaonboard, attitudeoptimal, roypoly}, (b) kinodynamic variants of sampling based planning \cite{kinodynamic} and (c) search based planning over lattice graphs \cite{sikangse3}. LQR trees explores the combination of sampling methods (\textit{i.e.} (b)) with trajectory optimization (\textit{i.e.} (a)) and successfully demonstrates in real-world dynamical systems \cite{lqrtrees}. However, it is an offline method to fill the entire state space with lookup policies that takes extremely long time to converge even for low-dimensional systems. In part inspired by LQR trees, in this paper, we explore an effective approach to combining trajectory optimization (\textit{i.e.} (a)) with search-based planning (\textit{i.e.} (c)) to develop an online planner and demonstrate it on a quadrotor performing aggressive flight.


\begin{figure}[h]
    \center
    \def\svgwidth{0.6\columnwidth}
    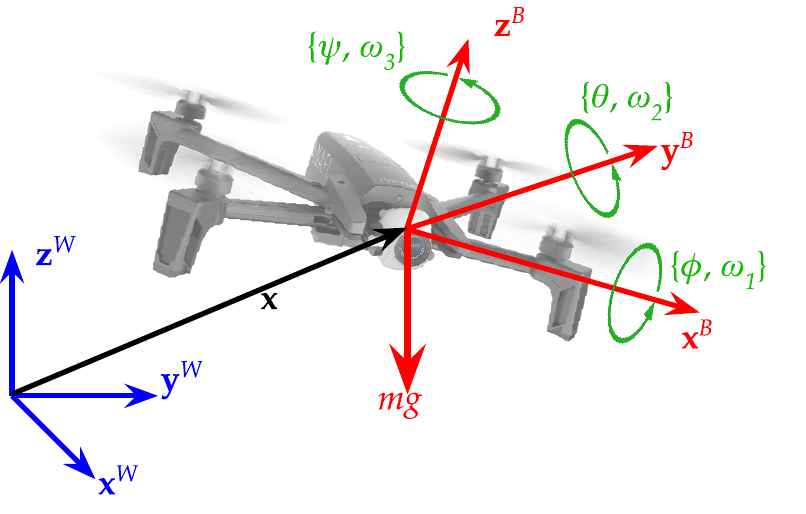
    \vspace{-0.2cm}
    \caption{\small Inertial (blue, superscript $W$) and the body fixed frame (red, superscript $B$) of the quadrotor. Origin of body fixed frame is at the quadrotor's center of mass. The direction of roll $\phi$, pitch $\theta$ and yaw $\psi$ and the corresponding angular velocities are shown in green.}
    \label{fig:quad}
\vspace{-0.7cm}
\end{figure}

To appreciate the potential of interleaving graph search and trajectory optimization, it is important to understand the trade-offs. Search-based planning has global reasoning capabilities and has proven successful in solving numerous robot motion planning problems \cite{anytime, maxara}. Despite that, planning for complex dynamical systems using search-based techniques still remains an uncharted area due to the challenge of discretizing an inherently continuous class of systems. For instance, when planning for a quadrotor with attitude constraints, the state space should contain all the pose variables and their finite derivatives to ensure kinodynamic feasibility. One way to deal with this high-dimensional search is to sparsely discretize the action space which impedes the planner's completeness guarantees. Consequently, trajectory optimization is a standard choice to deal with continuous actions and exploit the dynamic capabilities of the system but these are local methods and do not solve the full planning problem \cite{scaramuzzaonboard, attitudeoptimal, estquad}. 

Our contribution in this work is the novel framework called INSAT: INterleaved Search And Trajectory optimization for fast, global kinodynamic planning for aggressive quadrotor flight with completeness guarantees. The key idea behind our framework is (a) to identify a low-dimensional manifold, (b) perform a search over a grid-based graph that discretizes this manifold, (c) while searching the graph, utilize high-dimensional trajectory optimization to compute the cost of partial solutions found by the search. As a result, the search over the lower-dimensional graph decides what trajectory optimizations to run and with what seeds, while the cost of solution from the trajectory optimization drives the search in the lower-dimensional graph until a feasible high-dimensional trajectory from start to goal is found.
This paper is structured as follows: we discuss the related work in Sec. \ref{sec:relwork} and summarize the differential flatness property of the quadrotor which helps us to lift the low-dimensional discrete trajectory to high dimension in Sec. \ref{sec:prelim}. We formally define our problem in Section \ref{sec:prob} and describe our proposed method with its guarantees on completeness in Sec. \ref{sec:topus}. Finally, we show the experimental results in Sec. \ref{sec:results}, and conclude with future directions in Sec. \ref{sec:conc}. The code used in this work is open-sourced at \textproc{\url{https://github.com/nrkumar93/insat}}. 

\section{Related Work}
\label{sec:relwork}
Polynomial trajectory planning \cite{roypoly} jointly optimizes the pieces of a piecewise polynomial trajectory for flat systems \cite{murraydf} in a numerically robust and unconstrained fashion. It is a sequential method that uses a piecewise linear path as a seed for polynomial trajectory generation. Consequently, they do not handle attitude constraints for narrow gaps or perform global reasoning in case a part of the seed plan is invalid. Several prior works have demonstrated aggressive maneuvers for quadrotors that pass through narrow gaps \cite{scaramuzzaonboard, attitudeoptimal, estquad} but, instead of solving the planning problem, those works focus on trajectory optimization with given attitude constraints. Those constraints are often hand-picked beforehand or obtained using gap detection algorithms which only works for specific cases. 

Aggressive quadrotor planning for large environments typically involves building a safe corridor \cite{sikangsfc} such as convex decomposition of the free space \cite{tedrakeiris, tedrakeminp}. These methods do not deal with attitude constraints and hence there is no guarantee on planner's completeness when the robot has to aggressively negotiate a narrow gap. Liu's work \cite{sikangse3, sikanglqmt} on lattice search with predefined primitives for aggressive flight is the first method that attempts to incorporate quadrotor shape and dynamics in planning for large environments. It uses search-based methods to synthesize a plan over the lattice. However, lattice search suffers from the curse of dimensionality and their performance significantly depends on the choice of discretization for the state and the action space. Barring the interplay of low and high-dimensional search, our work is reminiscent of Theta* \cite{thetastar} as both the methods proceed by rewiring each successor to the best ancestor possible. However, Theta* is a planning algorithm designed specifically for 2D and 3D grid search and not applicable to higher dimensional planning like ours. 

Sampling-based robot motion planning has a rich history owing to their 
simplicity and scalability to higher dimensions \cite{prm, rrtstar}. But for kinodynamic planning, they rely on the ``steer'' operator which is often not efficient to compute \cite{kinodynamic}. They also suffer from the narrow passage problem \cite{narrowprm}, take longer time to converge to a good quality path and have unreliable intermediate path quality \cite{rrtstar}. Despite that, sampling-based trajectory optimization methods like LQR trees \cite{lqrtrees} with very high convergence time have enjoyed success and even been applied to hybrid systems \cite{mylqr}. These methods focus on the conditions for guaranteed execution based on the geometry of the trajectory funnels and the obstacles and even demonstrate it on a spherical quadrotor \cite{realtimelqr}. However, deriving such relations become extremely hard or almost impossible if the quadrotor is approximated as an ellipsoid. 



\section{Differential Flatness and Control of a Quadrotor}
\label{sec:prelim}

The quadrotor dynamics with four inputs (net thrust and the body moment about each axis) is differentially flat \cite{mellingerflat}. In other words, the states and inputs can be written as algebraic functions of the so-called flat outputs, $x$, $y$, $z$, and $\psi$ (yaw) and their derivatives. However, since the yaw is decoupled and does not affect the system dynamics, we do not consider it during planning. The Newton's equation of motion governing the acceleration of center of mass and the angular velocity of a standard quadrotor in terms of the flat outputs are
\eqnvspace
\begin{equation}
    m\ddot{\textbf{x}} = -mg\textbf{z}^W + f^B\textbf{z}^B
\label{eq:flatacc}
\eqnvspace
\end{equation}
\begin{equation}
    \boldsymbol{\omega}^B = 
    \begin{bmatrix}
        \omega_2 \\
        -\omega_1 \\
        \omega_3
    \end{bmatrix}=\frac{1}{f^B}
    \begin{bmatrix}
        1 & 0 & 0 \\ 
        0 & 1 & 0 \\ 
        0 & 0 & 0 
    \end{bmatrix}\textbf{R}^{-1}\dddot{\textbf{x}}
\label{eq:angvel}
\eqnvspace
\end{equation}
where $\textbf{x}$ is the position vector of the robot's center of mass in the inertial frame, $m$ is its mass, $g$ is the acceleration due to gravity, \textbf{R} describes the rotation of the body frame $B$ with respect to the inertial frame $W$, $\boldsymbol{\omega}^B$ and $f^B$ are the angular velocity vector and net thrust in the body-fixed coordinate frame (Fig. \ref{fig:quad}). $\textbf{z}^B$ is the unit vector aligned with the axis of the four rotors and indicates the direction of thrust, while $-\textbf{z}^W$ is the unit vector expressing the direction of gravity. 

The flatness property lets us calculate the quadrotor's orientation from the flat outputs and its derivatives. We make a useful observation from Eq. \ref{eq:flatacc} that the quadrotor can only accelerate in the direction of thrust and hence the attitude (roll and pitch) is constrained given the thrust vector. This constraint mapping is invertible and hence we can recover the direction of acceleration from attitude. In Sec. \ref{sec:trajopt}, we will describe and explicitly derive how the magnitude of acceleration is calculated by getting rid of the free variable $f^B$ in Eq. \ref{eq:flatacc}.

Following \cite{andreaindep}, we use triple integrator dynamics with jerk input for quadrotor planning. Trajectory segments consisting of three polynomial functions of time, each specifying the independent evolution of $x$, $y$, $z$, is used for quadrotor planning between two states in the flat output space \cite{cowlingpoly, roypoly, andreaindep}. As the desired trajectory and its derivatives are sufficient to compute the states and control inputs in closed form, they serve as a simulation of the robot’s motion in the absence of disturbances. This powerful capability is enabled by differential flatness that eliminates the need for iterated numerical integration of equations of motion, or a search over the space of inputs during each iteration of the planning algorithm.




\section{Problem Statement}
\label{sec:prob}
Let $\boldsymbol{\sigma}$ denote the translational variables of the quadrotor including its position, velocity, acceleration and jerk, \(\boldsymbol{\sigma} = [\textbf{x}\intercal, \textbf{\.x}\intercal, \textbf{\"x}\intercal, \dddot{\textbf{x}}\intercal]\intercal \in \mathbb{R}^{12}\). The 3D pose of the quadrotor is given by the position of its center of mass $\textbf{x}=[x, y, z]\intercal$ and orientation (in Euler angles) $\boldsymbol{\Theta}=[\phi, \theta, \psi]\intercal$ in the inertial frame. Given (a) an initial state $\textbf{s}_0 = [\boldsymbol{\sigma}\intercal_0, \boldsymbol{\Theta}_0\intercal, (\boldsymbol{\omega}_0^{B})\intercal, (\boldsymbol{\alpha}_0^B)\intercal]\intercal$ where $\boldsymbol{\omega}^{B}$ and $\boldsymbol{\alpha}^B$ are the angular velocity and angular acceleration of the body frame $B$, (b) a goal region $\mathcal{X}^{goal}$, (c) the planning space $\mathcal{X}$ with the obstacles $\mathcal{X}^{obs}$, the task is to find an optimal trajectory  \(\boldsymbol{\sigma}^*(t) = [\textbf{x}^*(t)\intercal, \textbf{\.x}^*(t)\intercal, \textbf{\"x}^*(t)\intercal, \dddot{\textbf{x}}^*(t)\intercal]\intercal\) according to Eq. \ref{eq:obj}, where $\textbf{x}^*(t) \in \mathcal{X} \setminus \mathcal{X}^{obs}$, $t \in [0, T]$ or the corresponding control inputs $\textbf{u}^*(t)$, $t \in [0, T]$. $\mathcal{X}^{obs}$ represents all the configurations of the robot that are in collision (Sec. \ref{sec:collfeas}) with its shape taken into consideration. 

For aggressive flight, the dynamical constraints of the quadrotor in terms of thrust and torques that can be supplied by the motors have to be satisfied while planning. Using the differential flatness property, these control saturation can be converted to componentwise box constraints on velocity, acceleration and jerk on each axis independently \cite{andreaeffiprim} as \( |\textbf{\.x}(t)| \preceq \textbf{\.x}_{max}, |\textbf{\"x}(t)| \preceq \textbf{\"x}_{max}, |\dddot{\textbf{x}}(t)| \preceq \dddot{\textbf{x}}_{max}\). Thus the time-optimal path-planning for aggressive quadrotor flight can be cast as the following optimization problem: 
\eqnvspace
\begin{equation}
\begin{aligned}
    \min_{\textbf{x}(t), \textbf{u}(t), T}  \quad & J_{total} = \int_0^T \| \dddot{\textbf{x}}(t)\|^2 + \gamma T \\ 
\textrm{s.t.} \quad & \textbf{\.x} = F(\textbf{x}, \textbf{u}), \\
& \textbf{x}(0) = \textbf{x}_0, \\
& \textbf{x}(T) \in \mathcal{X}^{goal}, \\
& |\textbf{\.x}(t)| \preceq \textbf{\.x}_{max}, |\textbf{\"x}(t)| \preceq \textbf{\"x}_{max}, |\dddot{\textbf{x}}(t)| \preceq \dddot{\textbf{x}}_{max}\\
& \textbf{x}(t) \in \mathcal{X} \setminus \mathcal{X}^{obs}, \textbf{u} \in \textbf{U} \ \  \forall t \in [0, T]
\end{aligned}
\label{eq:obj}
\eqnvspace
\end{equation}
where $F$ and \textbf{U} denote the quadrotor dynamics and the set of all attainable control vectors, $J_{total}$ is total cost of the trajectory and $\gamma$ is the penalty to prioritize control effort over execution time $T$. It is sufficient to find the optimal trajectory purely in terms of translational variables as the reminder of state can be recovered using the results of differential flatness.



\section{Motion Planning For Aggressive Flight}
\label{sec:topus}
Our trajectory planning framework consists of two \textit{overlapping} modules: a grid-based graph search planner and a trajectory optimization routine. These two methods are interleaved to combine the benefits of former's ability to search non-convex spaces and solve combinatorial parts of the problem and the latter's ability to obtain a locally optimal solution not constrained to the discretized search space. We provide analysis (Sec. \ref{sec:analysis}) and experimental evidence (Sec. \ref{sec:results}) that interleaving provides a superior alternative in terms of quality of the solution and behavior of the planner than the naive option of running them in sequence \cite{roypoly}. 

We begin by providing a brief overview of the polynomial trajectory optimization setup. This will be followed by the description of the INSAT framework and how it utilizes graph search and polynomial trajectory generation. We then analyse INSAT's guarantees on completeness.

\subsection{Attitude Constrained Joint Polynomial Optimization}
\label{sec:trajopt}
To generate a minimum-jerk and minimum-time trajectory, the polynomial generator should compute a thrice differentiable trajectory that guides the quadrotor from an initial state to a partially defined final state by respecting the spatial and dynamic constraints while minimizing the cost function given in Eq. \ref{eq:obj}. For quadrotors, it is a common practice to consider triple integrator dynamics and decouple the trajectory generation \cite{andreaindep, sikangse3} into three independent problems along each axis. However, for attitude constrained flight, although the dynamic inversion provided by the flatness property aids in determining the direction of acceleration from the desired attitude, the corresponding magnitude cannot be computed by axis independent polynomial optimization. We note from Eq. \ref{eq:flatacc} that the thrust supplied by the motors $f^B$ is a free variable which can be eliminated to deduce a constraint relationship between the components of the acceleration vector $\ddot{\textbf{x}}$ and the direction of thrust in body frame $\textbf{z}^B$ as follows
\eqnvspace
\begin{equation}
     \frac{\ddot{\textbf{x}}_x}{\textbf{z}^B_x}= \frac{\ddot{\textbf{x}}_y}{\textbf{z}^B_y} = \frac{\ddot{\textbf{x}}_z - g}{\textbf{z}^B_z}
     \label{eq:constraint}
\eqnvspace
\end{equation}
where $\ddot{\textbf{x}}_i$ and $\textbf{z}_i^B$ are the axis-wise components of acceleration and thrust vector. Rearranging the terms in Eq. \ref{eq:constraint} provides a linear constraint on acceleration independent of the thrust
\eqnvspace
\begin{equation}
\underbrace{\begin{bmatrix}
-\textbf{z}^B_y & \textbf{z}^B_x & 0\\
-\textbf{z}^B_z & 0 & \textbf{z}^B_y \\
0 & -\textbf{z}^B_z & \textbf{z}^B_y
\end{bmatrix}}_\textbf{W}
\begin{bmatrix}
    \ddot{\textbf{x}}_x \\ 
    \ddot{\textbf{x}}_y \\ 
    \ddot{\textbf{x}}_z
\end{bmatrix} = 
\underbrace{\begin{bmatrix}
    0 \\ 
    g\textbf{z}^B_x \\
    g\textbf{z}^B_y
\end{bmatrix}}_\textbf{d}
    \label{eq:constraintmatrix}
\eqnvspace
\end{equation}
\begin{equation}
    \textbf{W}\ddot{\textbf{x}} = \textbf{d}
    \label{eq:constraintshort}
\eqnvspace
\end{equation}
We incorporate the constraint derived above in the joint polynomial optimization method introduced in \cite{roypoly} to find a sequence of polynomials through a set of desired attitude constrained waypoints. Thus, the first term of the cost function in Eq. \ref{eq:obj} can be transformed into product of coefficients of $M$ polynomials and their Hessian with respect to $N$ coefficients per polynomial thereby forming a quadratic program (QP)
\eqnvspace
\begin{equation}
    J_{control} = \int_0^T \| \dddot{\textbf{x}}(t)\|^2 = \textbf{p}\intercal\textbf{H}\textbf{p}
    \label{eq:hessian}
\eqnvspace
\end{equation}
where \textbf{p}$\in \mathbb{R}^{MN}$ represents all the polynomial coefficients grouped together and $\textbf{H}$ is the block Hessian matrix with each block corresponding to a single polynomial. Note that the integrand encodes the sequence of polynomial segments as opposed to just one polynomial and each block of the Hessian matrix is a function of time length of the polynomial segment. We omit the details for brevity and defer the reader to \cite{roypoly} for a comprehensive treatment. Following \cite{roypoly}, the requirement to satisfy the position constraints and derivative continuity is achieved by observing that the derivatives of the trajectory are also polynomials whose coefficients depend linearly on the coefficients of the original trajectory. In our case, in addition to position and continuity constraints we have to take the attitude constraints into account via acceleration using Eq. \ref{eq:constraintshort}.
\eqnvspace
\begin{equation}
    \textbf{A}\textbf{p} = \begin{bmatrix}
        \dot{\textbf{x}} \\ \ddot{\textbf{x}} \\ \dddot{\textbf{x}}
    \end{bmatrix} \implies \textbf{A}\textbf{p} = \underbrace{\begin{bmatrix}
        \textbf{b} \\
    \textbf{W}^{-1}\textbf{d}
    \end{bmatrix}}_\text{\textbf{c}} \implies \textbf{p} = \textbf{A}^{-1}\textbf{c}
    \label{eq:fullconstraint}
\eqnvspace
\end{equation}
where the matrix \textbf{A} maps the coefficients of the polynomials to their endpoint derivatives and \textbf{b} contains all other derivative values except acceleration which is obtained using Eq. \ref{eq:constraintshort}. Using Eq. \ref{eq:fullconstraint} in Eq. \ref{eq:hessian} 
\eqnvspace
\begin{equation}
    J_{control} = \textbf{c}\intercal\textbf{A}^{-}\intercal\textbf{H}\textbf{A}^{-1}\textbf{c} 
    \label{eq:unconstrained}
\eqnvspace
\end{equation}

Note that due to the interdependent acceleration constraint (Eq. \ref{eq:constraintmatrix}) imposed at the polynomial endpoints, we lost the ability to solve the optimization independently for each axis. Nevertheless, the key to the efficiency of our approach lies in the fact that solving a QP like Eq. \ref{eq:hessian} subject to linear constraints in Eq. \ref{eq:fullconstraint} or in their unconstrained format in Eq. \ref{eq:unconstrained} is incredibly fast and robust to numerical instability. Thus the total jerk and time cost $J_{total}$ to be minimized becomes 
\eqnvspace
\begin{equation}
    J_{total} = \underbrace{\textbf{c}\intercal\textbf{A}^{-}\intercal\textbf{H}\textbf{A}^{-1}\textbf{c}}_{J_{control}} + \underbrace{\gamma \sum_{i=1}^M T_i}_{J_{time}}
    \label{eq:totalcost}
\eqnvspace
\end{equation}
where $T_i$ expresses the time length of the $i$th polynomial. As mentioned before, the Hessian depends on the choice of time length of the polynomial segment and hence the overall cost is minimized by running a gradient descent on $T_i$ and evaluating $J_{control}$ corresponding to a particular $T_i$. 


\subsection{INSAT: Interleaving Search And Trajectory Optimization}
\label{sec:insatalg}
To plan a trajectory that respects system dynamics and controller saturation, and simultaneously reason globally over large non-convex environments, it is imperative to maintain the combinatorial graph search tractable. To this end, we consider a low-dimensional space $\mathcal{X}_L$ (5D) comprising $\{\textbf{x}^\intercal, \phi, \theta \}$. The discrete graph search runs in $\mathcal{X}_L$ which typically contains variables of the state whose domain is non-convex. It then seeds the trajectory optimization, such as the one in Sec. \ref{sec:trajopt}, in the high-dimensional space $\mathcal{X}_H$ (12D) comprising $\{ \textbf{x}\intercal, \textbf{\.x}\intercal,$ $\textbf{\"x}\intercal, \dddot{\textbf{x}}\intercal\}$, to in turn obtain a better estimate of the cost-to-come value of a particular state for the graph search. The subscripts $L$ and $H$ refer to the low and high-dimensional states.

Alg. \ref{alg:topus} presents the pseudocode of INSAT. Let $\textbf{s}_L\in \mathcal{X}_L$ and $\textbf{s}_H\in \mathcal{X}_H$ be the low-dimensional and high-dimensional state. The algorithm takes as input the high-dimensional start and goal states $\textbf{s}_H^{start}, \textbf{s}_H^{goal}$ and recovers their low-dimensional counterparts $\textbf{s}_L^{start}, \textbf{s}_L^{goal}$ (lines \ref{line:ldrecstart}-\ref{line:ldrecend}). The low-dimensional free space $\mathcal{X}_L\setminus\mathcal{X}^{obs}$ is discretized to build a graph $\mathcal{G}_L$ to search. To search in $\mathcal{G}_L$, we use weighted A* (WA*)\cite{pohlwastar} which maintains a priority queue called OPEN that dictates the order of expansion of the states and the termination condition based on \textproc{Key($\textbf{s}_L$)} value (lines \ref{line:key}, \ref{line:term}). Alg. \ref{alg:topus} maintains two functions: cost-to-come $g(\textbf{s}_L)$ and a heuristic $h(\textbf{s}_L)$. $g(\textbf{s}_L)$ is the cost of the current path from the start state to $\textbf{s}_L$ and $h(\textbf{s}_L)$ is an underestimate of the cost of reaching the goal from $\textbf{s}_L$. WA* initializes OPEN with $\textbf{s}_L^{state}$ (line \ref{line:init}) and keeps track of the expanded states using another list called CLOSED (line \ref{line:closed}). 

\setlength{\textfloatsep}{4pt}
\begin{algorithm}
\begin{algorithmic}[1]

\Procedure{Key}{$\textbf{s}_L$} \label{line:key}
\State \textbf{return} $g(\textbf{s}_L) + \epsilon*h(\textbf{s}_L)$
\EndProcedure


\Procedure{GenerateTrajectory}{$\textbf{s}_L$, $\textbf{n}_L$}
\State $\textbf{z}^B = [sin(\textbf{n}_L^\theta),\  -cos(\textbf{n}_L^\theta)sin(\textbf{n}_L^\phi),\ cos(\textbf{n}_L^\phi)cos(\textbf{n}_L^\theta) ]^\top $ 
\State $\textbf{n}_H^{\ddot{\textbf{x}}}$ = $\textbf{W}^{-1}\textbf{d}$ \Comment{Differential flatness Eq. \ref{eq:constraintmatrix}}\label{line:dfmap2}
\State $\textbf{c} = [(\textbf{s}_H^{start})\intercal, (\textbf{n}_H)\intercal, (\textbf{s}_H^{goal})\intercal]\intercal$ \Comment{Eq. \ref{eq:fullconstraint}}\label{line:trajtop11}
\State $\textbf{n}_H(t)$ = $TrajectoryOptimizer(\textbf{c})$ \Comment{Eq. \ref{eq:totalcost}}\label{line:trajtop12}
\State \textbf{if} $\textbf{n}_H(t_{\textbf{n}_H}).IsCollisionFree() \bigwedge$\Comment{Sec. \ref{sec:collfeas}}
\Statex[1.33] $\>$ $\textbf{n}_H(t_{\textbf{n}_H}).IsInputFeasible()$ \textbf{then} \Comment{Sec. \ref{sec:infeas}}\label{line:feas1}
\State $\>$ \textbf{return} $\textbf{n}_H(t)$ \label{line:retbasicphase}
\State \textbf{else} \label{line:repair}
\State $\>$ \textbf{for} $\textbf{m}_H \in Waypoints(\textbf{s}_H(t))$ \textbf{do} \label{line:wpsearch}
\State $\>$ $\>$ $\textbf{c} = [(\textbf{m}_H)\intercal, (\textbf{n}_H)\intercal, (\textbf{s}_H^{goal})\intercal]\intercal$ \Comment{Eq. \ref{eq:fullconstraint}}\label{line:trajtop21}
\State $\>$ $\>$ $\textbf{r}_H(t)$ = $TrajectoryOptimizer(\textbf{c})$ \Comment{Eq. \ref{eq:totalcost}}\label{line:trajtop22}
\State $\>$ $\>$ \textbf{if} $\textbf{r}_H(t_{\textbf{n}_H}).IsCollisionFree() \bigwedge$ \Comment{Sec. \ref{sec:collfeas}}
\Statex[1.79] $\>$ $\textbf{r}_H(t_{\textbf{n}_H}).IsInputFeasible()$ \textbf{then} \Comment{Sec. \ref{sec:infeas}}\label{line:feas2}
\State $\>$ $\>$ $\>$ $\textbf{n}_H(t) = Concatenate(\textbf{m}_H(t), \textbf{r}_H(t))$\label{line:concat}
\State $\>$ $\>$ $\>$ $RelaxConstraintsAndWarmStart(\textbf{n}_H(t))$ \label{line:warm}
\State $\>$ $\>$ $\>$ \textbf{return} $\textbf{n}_H(t)$ \label{line:retrepair}
\State \textbf{return} Tunnel traj. w/ discrete $\infty$ cost \Comment{Sec. \ref{sec:completeness}}\label{line:tunnel}
\EndProcedure

\Procedure{Main}{$\textbf{s}_H^{start}, \textbf{s}_H^{goal}$}
\State $(\textbf{s}_L^{start})^\textbf{x} = (\textbf{s}_H^{start})^\textbf{x}$; \hspace{0.1cm} $(\textbf{s}_L^{goal})^\textbf{x} = (\textbf{s}_H^{goal})^\textbf{x}$ \label{line:ldrecstart}
\State $(\textbf{s}_L^{start})^{\phi, \theta} =$ Obtain from $(\textbf{s}_H^{start})^{\ddot{\textbf{x}}}$ \Comment{Eq. \ref{eq:flatacc}}
\State $(\textbf{s}_L^{goal})^{\phi, \theta} =$ Obtain from $(\textbf{s}_H^{goal})^{\ddot{\textbf{x}}}$ \Comment{Eq. \ref{eq:flatacc}}\label{line:ldrecend}
\State $\forall \textbf{s}_L, g(\textbf{s}_L) = \infty$; $g(\textbf{s}_L^{start}) = 0$
\State Insert $\textbf{s}_L^{start}$ in OPEN with \textproc{Key}($\textbf{s}_L^{start}$) \label{line:init}
\State \textbf{while} \textproc{Key}($\textbf{s}_L^{goal}$) $< \infty$ \textbf{do} \label{line:term}
\State $\>$ $\textbf{s}_L =$ OPEN.$pop()$ \label{line:pq} 
\State $\>$ \textbf{for} $\textbf{s}^\prime_L \in Succ(\textbf{s}_L)$ \textbf{do}  \label{line:succ}
\State $\>$ $\>$ $\textbf{n}_L = SoftCopy(\textbf{s}^\prime_L)$\label{line:softcopy}
\State $\>$ $\>$ \textbf{if} $\textbf{n}_L \in $ CLOSED \textbf{then} \label{line:closed}
\State $\>$ $\>$ $\>$ $\textbf{n}_L = DeepCopy(\textbf{s}^\prime_L)$; $g(\textbf{n}_L) = \infty$ \Comment{Sec. \ref{sec:completeness}}\label{line:deepcopy}
\State $\>$ $\>$ $\textbf{n}_H(t)$ = \textproc{GenerateTrajectory}($\textbf{s}_L$, $\textbf{n}_L$)
\State $\>$ $\>$ \textbf{if} $\textbf{n}_H(t).IsCollisionFree() \bigwedge$ \Comment{Sec. \ref{sec:collfeas}}
\Statex[1.7] $\>$ $\>$ $\textbf{n}_H(t).IsInputFeasible()$ \textbf{then} \Comment{Sec. \ref{sec:infeas}}\label{line:feas3}
\State $\>$ $\>$ $\>$ $g(s_L^{goal}) = J_{total}(\textbf{n}_H(t))$ \Comment{Eq. \ref{eq:obj}}
\State $\>$ $\>$ $\>$ Insert/Update $\textbf{n}_L$ in OPEN with \textproc{Key}($s_L^{goal}$)\label{line:goalcand}
\State $\>$ $\>$ \textbf{if}  $J_{total}(\textbf{n}_H(t_{\textbf{n}_H})) < g(\textbf{n}_L)$ \textbf{then} \Comment{Eq. \ref{eq:obj}}
\State $\>$ $\>$ $\>$ $g(\textbf{n}_L) = J_{total}(\textbf{n}_H(t_{\textbf{n}_H}))$ \Comment{Eq. \ref{eq:obj}}
\State $\>$ $\>$ $\>$ Insert/Update $\textbf{n}_L$ in OPEN with \textproc{Key}($\textbf{n}_L$)
\EndProcedure
\end{algorithmic}
\caption{INSAT}
\label{alg:topus}
\end{algorithm}

\begin{figure*}[h]
    \center
    \def\svgwidth{\textwidth}
    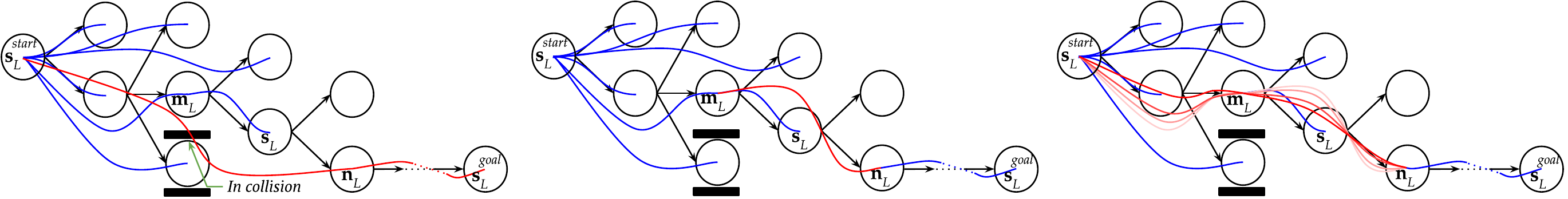
    \caption{\small Graphical illustration of the \textproc{GenerateTrajectory()} function of INSAT (Alg \ref{alg:topus}). Here the state $\textbf{s}_L$ is expanded and a trajectory is optimized for its successor $\textbf{n}_L$. LEFT: At first, the optimizer tries to find a trajectory directly from start to goal via $\textbf{n}_H$ ($\textbf{n}_L$'s high-dimensional counterpart) as shown in red (lines \ref{line:trajtop11}-\ref{line:trajtop12}). CENTER: If the portion of the trajectory from the first attempt up to $\textbf{n}_L$ is input infeasible or in collision (as in LEFT), then instead of the start state the earliest possible waypoint $\textbf{m}_H$ ($\textbf{m}_L$'s high-dimensional counterpart) on the high-dimensional trajectory $\textbf{s}_H(t)$ is selected and a new trajectory segment is incrementally optimized (shown in red) as in lines \ref{line:wpsearch}-\ref{line:feas2}. RIGHT: Once a set of collision free and feasible trajectory segments are found, we refine the trajectory by relaxing all the waypoint and derivative constraints (convergence shown with different shades of red). Note that this stage can consist of several polynomials being jointly optimized, however, the convergence is extremely fast due to warm starting (line \ref{line:warm}).}
    \label{fig:topus}
\vspace{-0.7cm}
\end{figure*}

A graphical illustration of the algorithm is provided in Fig. \ref{fig:topus}. Each time the search expands a state $\textbf{s}_L$, it removes $\textbf{s}_L$ from OPEN and generates the successors as per the discretization (lines \ref{line:pq}-\ref{line:softcopy}). For every low-dimensional successor $\textbf{n}_L$, we solve a trajectory optimization problem described in Sec. \ref{sec:trajopt} to find a corresponding high-dimensional trajectory from start to goal via $\textbf{n}_H$ (lines \ref{line:trajtop11}-\ref{line:trajtop12}, Fig \ref{fig:topus}). Note that the trajectory optimization is performed in the space of translational variables but $\textbf{n}_L$ specifies an attitude requirement. So prior to trajectory optimization, we utilize the differential flatness property to transform the attitude of the quadrotor to an instantaneous direction and magnitude of acceleration $\textbf{n}_H^{\ddot{\textbf{x}}}$ to be satisfied (line \ref{line:dfmap2}, Eq. \ref{eq:constraintmatrix}). The trajectory optimization output $\textbf{n}_H(t)$ is checked for collision and control input feasibility (line \ref{line:feas1}, Sec. \ref{sec:feas}). If the optimized trajectory $\textbf{n}_H(t)$ is in collision or infeasible (Fig. \ref{fig:topus}-Left), the algorithm enters the repair phase (lines \ref{line:repair}-\ref{line:retrepair}). 

The repair phase is same as the first call to the optimizer except that instead of the start state $\textbf{s}_H^{start}$, we iterate over the waypoints $\textbf{m}_H$ (line \ref{line:wpsearch}) of the parent state's trajectory $\textbf{s}_H(t)$ in order (lines \ref{line:wpsearch}-\ref{line:feas2}, Fig. \ref{fig:topus}-Center). It has to be noted that the computational complexity of trajectory optimization QP is same for both the initial attempt and the repair phase as the sequence of polynomials from $\textbf{s}_H^{start}$ to $\textbf{m}_H$ is unmodified. Upon finding the state $\textbf{m}_H$ which enables a high-dimensional feasible trajectory from start to goal via $\textbf{n}_H$, the full trajectory $\textbf{n}_H(t)$ is constructed by concatenating $\textbf{m}_H(t)$ up to $\textbf{m}_H$ and the newly repaired trajectory, $\textbf{r}_H(t)$, starting from $\textbf{m}_H$ (line \ref{line:concat}). The final trajectory is obtained by warm starting the optimization with the trajectory $\textbf{n}_H(t)$ as the seed and relaxing all the waypoint and derivative constraints (Fig. \ref{fig:topus}-Right) until convergence or trajectory becoming infeasible, whichever occurs first. We remark that, within \textproc{GenerateTrajectory()}, the trajectory is checked for collision and feasibility only until the waypoint $\textbf{n}_H$ indicated by time $t_{\textbf{n}_H}$ (lines \ref{line:feas1}, \ref{line:feas2}) although the trajectory connects all the way from start to goal via $\textbf{n}_H$. The validity of the full trajectory is checked in \textproc{Main()} (line \ref{line:feas3}) to be considered as a potential goal candidate (line \ref{line:feas3}-\ref{line:goalcand}).

\subsection{Completeness Analysis of INSAT}
\label{sec:analysis} 
\label{sec:completeness}
We import the notations $\mathcal{X}^{obs}$, $\mathcal{G}_L$ from \ref{sec:insatalg}. $\mathcal{G}_L = (\mathcal{V}_L, \mathcal{E}_L)$ where $\mathcal{V}_L$ and $\mathcal{E}_L$ are set of vertices and edges, $\mathcal{X}^{free}_L=\mathcal{X}_L\setminus\mathcal{X}^{obs}$, $\boldsymbol{\tau}_\mathcal{G}$ be any path in $\mathcal{G}_L$, $\boldsymbol{\tau}_L(t)$ be the low-dimensional trajectory and $\boldsymbol{\tau}_H(t)$ be the high-dimensional trajectory that is snap continuous.  

\textbf{Assumption (AS):} If there exists $\boldsymbol{\tau}_L(t) \in \mathcal{X}^{free}_L$ then there exists a corresponding path $\boldsymbol{\tau}_\mathcal{G}$ in $\mathcal{G}_L$
\eqnvspace
\begin{equation*}
    \boldsymbol{\tau}_\mathcal{G} = \{(v, v^{\prime}) \mid v, v^{\prime} \in \mathcal{V}_L, (v, v^{\prime}) \in \mathcal{E}_L, \mathcal{T}(v, v^{\prime}) \subseteq \mathcal{X}_L^{free} \}
    \label{eq:trajingraph}
\eqnvspace
\end{equation*}
where $\mathcal{T}(v, v^{\prime})$ is the tunnel around the edge $(v, v^{\prime})$ (Fig. \ref{fig:tunnel}).

\begin{theorem}
\label{theo:obs}
$\exists \ \boldsymbol{\tau}_H(t) \in \mathcal{X}^{free}_H \implies \exists \ \boldsymbol{\tau}_L(t) \in \mathcal{X}^{free}_L$
\end{theorem}

\begin{proof}
Using quadrotor's differential flatness all the variables of $\mathcal{X}_L$ can be recovered from the variables in $\mathcal{X}_H$. So the map $\mathcal{M}^H_{L}: \mathcal{X}_H \mapsto \mathcal{X}_L$ is a surjection. But $\mathcal{X}^{free}_H = \{\textbf{x}_H \in \mathcal{X}_H \mid \mathcal{M}^H_L(\textbf{x}_H) \in \mathcal{X}_L^{free}\}$ and hence the map $\mathcal{M}^{H(free)}_{L}: \mathcal{X}_H^{free} \mapsto \mathcal{X}_L^{free}$ is also a surjection.
\end{proof}

\begin{theorem}[Completeness]
If $\exists \ \boldsymbol{\tau}_H(t) \in \mathcal{X}^{free}_H$, then INSAT is guaranteed to find a $\boldsymbol{\tau}^\prime_H(t)\in\mathcal{X}^{free}_H$.
\end{theorem}

\begin{figure}[h]
    \center
    \def\svgwidth{0.4\columnwidth}
    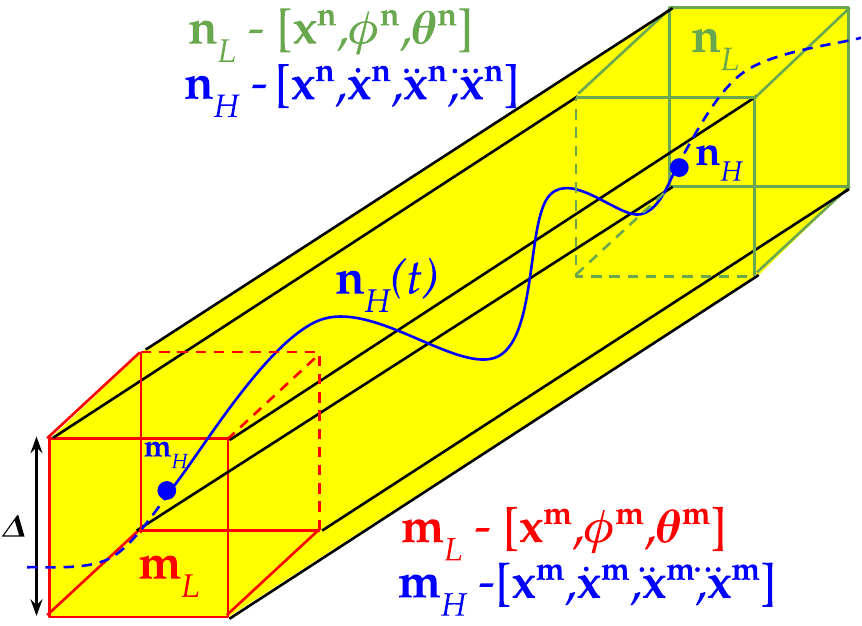
    \caption{\small Part of the high-dimensional trajectory $\textbf{n}_H(t)$ from $\textbf{s}_L^{start}$ to $\textbf{s}_L^{goal}$ via the expanded node $\textbf{m}_L$ and its successor $\textbf{n}_L$. The portion of $\textbf{n}_H(t)$ between $\textbf{m}_L$ and $\textbf{n}_L$ is guaranteed to lie within the tunnel (yellow) formed by $\textbf{m}_L$ and $\textbf{n}_L$ and is called as tunnel trajectory.}
    \label{fig:tunnel}
\vspace{-0.3cm}
\end{figure}

\begin{proof}
\textbf{Inference (IN):} If \textbf{AS} holds, it is enough to search $\mathcal{G}_L$ instead of $\mathcal{X}^{free}_L$. Then from Theorem. \ref{theo:obs}  we can deduce that there exists a $\boldsymbol{\tau}_\mathcal{G}$ in $\mathcal{G}_L$ if $\exists \ \boldsymbol{\tau}_H(t) \in \mathcal{X}^{free}_H$.


Thus to prove the completeness of INSAT, we have to show that Alg. \ref{alg:topus} finds a $\boldsymbol{\tau}^\prime_H(t) \in \mathcal{X}^{free}_H$ for any $\boldsymbol{\tau}_\mathcal{G}$ in $\mathcal{G}_L$ (\emph{i.e} converse of \textbf{IN}). We prove by induction. At $i$th step of INSAT, let $\mathcal{G}_L^{i} = (\mathcal{V}_L^i, \mathcal{E}_L^i)$ be the low-dimensional graph for which there exists a $\boldsymbol{\tau}_H^i(t) \in \mathcal{X}_H^{free}$ from $\textbf{s}_L^{start}$ to any $\textbf{s}_L^i \in \mathcal{V}_L^i$. The induction is to prove that, at $(i+1)$th step, after adding any number of nodes to get $\mathcal{G}_L^{i+1} = (\mathcal{V}_L^{i+1}, \mathcal{E}_L^{i+1})$, INSAT is guaranteed to find $\boldsymbol{\tau}_H^{i+1}(t) \in \mathcal{X}_H^{free}$ from $\textbf{s}_L^{start}$ to every $\textbf{s}_L^{i+1} \in \mathcal{V}_L^{i+1}$. Let $\textbf{m}_L^i \in \mathcal{V}_L^i$ be the node expanded at $(i+1)$th step from $\mathcal{G}_L^i$ to generate a successor $\textbf{n}_L^{i+1} \in \mathcal{V}^{i+1}_L$ and the graph $\mathcal{G}_L^{i+1}$. We know that $\textbf{m}^{i}_H(t) \in \mathcal{X}^{free}_H$. So even if the basic (lines \ref{line:trajtop11}-\ref{line:retbasicphase}) and the repair (lines \ref{line:repair}-\ref{line:retrepair}) phases fail (Sec. \ref{sec:insatalg}), Alg. \ref{alg:topus} falls back to finding the tunnel trajectory to concatenate with $\textbf{m}^{i}_H(t)$ (line \ref{line:tunnel}). The tunnel trajectory between $\textbf{m}^i_H$ and $\textbf{n}^{i+1}_H$ (i) is collision-free under \textbf{AS} (ii) satisfies the boundary pose and derivative constraints (iii) snap continuous. The existence of such a tunnel trajectory can be shown using trigonometric bases but it is beyond the scope of this proof. The ``base case'' of $\mathcal{G}_L^i, i=0$ with 1 node ($\textbf{s}^{start}_L$) is collision-free $\textbf{s}^{start}_H(t) \in \mathcal{X}^{free}_H$. And INSAT finds $\boldsymbol{\tau}_H^{i+1}(t) \in \mathcal{X}_H^{free}$ even at $(i+1)$th step. Hence, INSAT is a provably complete algorithm.
\end{proof}

\subsection{Trajectory Feasibility}
\label{sec:feas}
To plan for aggressive trajectories in cluttered environments, we approximate the shape of the quadrotor as an ellipsoid to capture attitude constraints and check for collision. During a state expansion, once the high-dimensional polynomial trajectory is found from the start to goal via a successor, it is checked for any violation of dynamics and control input (thrust and angular velocity) limits. 

\subsubsection{Input Feasibility}
\label{sec:infeas}
We use a recursive strategy introduced in \cite{andreaeffiprim} to check jerk input trajectories for input feasibility by binary searching and focusing only on the parts of the polynomial that violate the input limits. The two control inputs to the system are thrust and the body rate in the body frame. For checking thrust feasibility, the maximum thrust along each axis is calculated independently from acceleration (Eq. \ref{eq:flatacc}), by performing root-finding on the derivative of the jerk input polynomial trajectory. The maximum/minimum value among all the axes is used to check if it lies within the thrust limits. For body rate, its magnitude can be bounded as a function of the jerk and thrust (Eq. \ref{eq:angvel}). Using this relation, we calculate the body rate along the trajectory and check if it entirely lies within the angular velocity limits. Note that, in the implementation, these two feasibility tests are done in parallel. 

\subsubsection{Collision Checking}
\label{sec:collfeas}
We employ a two level hierarchical collision checking scheme. The first level checks for a conservative validity of the configuration and refines for an accurate collision check only if the first level fails. In the first level, we approximate the robot as a sphere and inflate the occupied cells of the voxel grid with its radius. This lets us treat the robot as a single cell and check for collision in cells along the trajectory. The second level follows the ellipsoid based collision checking that takes the actual orientation of the quadrotor into account \cite{sikangse3}. By storing the points of the obstacle pointcloud in a KD-tree, we are able to crop a subset of the points and efficiently check for collisions only in the neighborhood of the robot. 



\begin{figure*}
    \begin{subfigure}{\textwidth}
    \centering
        \includegraphics[width=\textwidth]{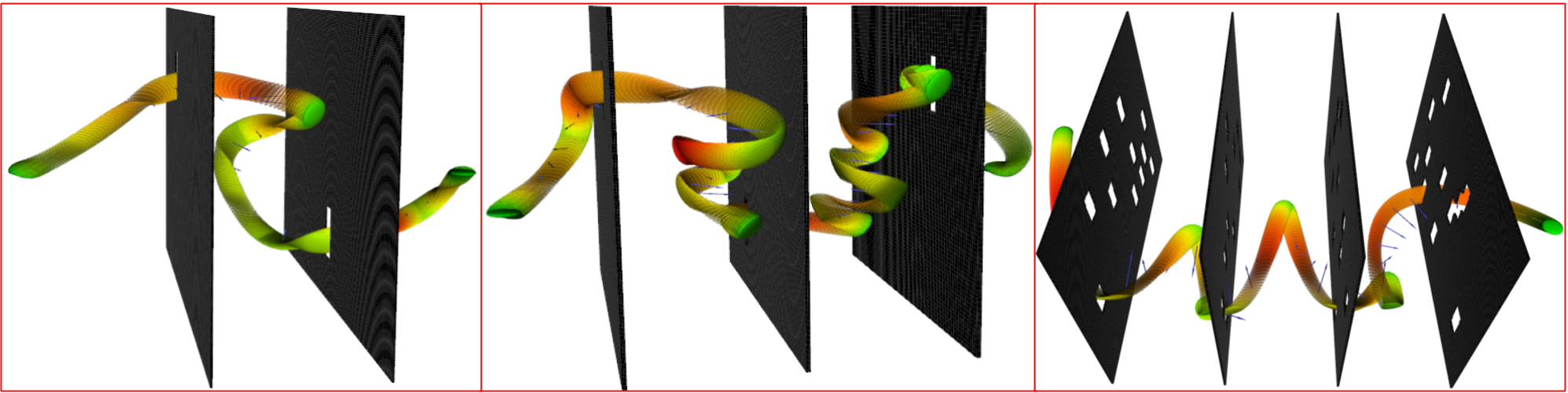}
    \caption{Side views. LEFT: \#Walls: 2, \#Holes/wall: 1. CENTER: \#Walls: 3, \#Holes/wall: 1. RIGHT: \#Walls: 4, \#Holes/wall: 11}
    \label{fig:walldemoa}
    \end{subfigure}
    \begin{subfigure}{\textwidth}
    \centering
        \includegraphics[width=\textwidth]{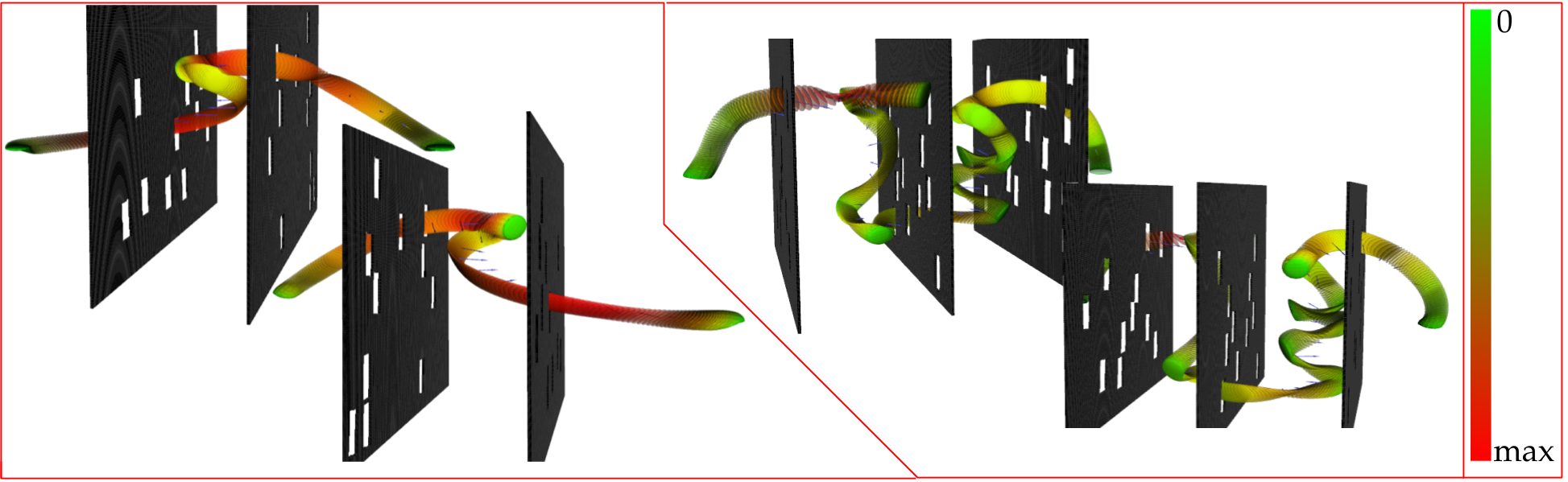}
    \caption{L\&R side views. LEFT: \#Walls: 2, \#Holes/wall: 11. CENTER: \#Walls: 3, \#Holes/wall: 11. RIGHT: Velocity magnitude}
    \label{fig:walldemob}
    \end{subfigure}
    \caption{\small Visualization of trajectory planned by the proposed method in the walls and windows environment. The environment consists of parallel walls with narrow holes (windows) smaller than the size of the quadrotor. The trajectory is represented as a sequence of ellipsoid which approximates the quadrotor's shape to plan in attitude constrained environments. The color gradient from green to red expresses the magnitude of instantaneous velocity while the arrows along the trajectory denote the magnitude and direction of acceleration. The dynamically stable agile behavior of the planner is analyzed in Sec. \ref{sec:structenv}.}
    \label{fig:walldemo}
\vspace{-0.3cm}
\end{figure*}

\begin{table*}[t]
\centering
\resizebox{\textwidth}{!}{%
\begin{tabular}{|c|c|c|c|c|c|c|c|c|c|c|}
\hline
\multirow{2}{*}{\#Walls} & \multirow{2}{*}{\begin{tabular}[c]{@{}c@{}}\#Windows\\ per wall\end{tabular}} & \multirow{2}{*}{Wall gap (m)} & \multicolumn{2}{c|}{Planning Time (s)} & \multicolumn{2}{c|}{Success Rate (\%)} & \multicolumn{2}{c|}{Solution Cost ($\times 10^5$)} & \multicolumn{2}{c|}{Execution Time (s)} \\ \cline{4-11} 
 &  &  & INSAT & Base A & INSAT & Base A & INSAT & Base A & INSAT & Base A \\ \hline
\multirow{6}{*}{2} & \multirow{2}{*}{1} & 5 & 19.37 $\pm$ 17.34 & 122.5 $\pm$ 88.44 & 100 & 36 & 3.7 & 7.18 & 9.12 $\pm$ 1.04 & 10.4 $\pm$ 1.18 \\ \cline{3-11} 
 &  & 9 & 29.76 $\pm$ 28.41 & 180.47 $\pm$ 93.22 & 100 & 24 & 5.5 & 7.48 & 9.42 $\pm$ 1.54 & 10.81 $\pm$ 1.71 \\ \cline{2-11} 
 & \multirow{2}{*}{5} & 5 & 46.58 $\pm$ 49.76 & 97.15 $\pm$ 70.01 & 100 & 100 & 6 & 6.16 & 6.93 $\pm$ 1.89 & 8.97 $\pm$ 2.25 \\ \cline{3-11} 
 &  & 9 & 22.07 $\pm$ 38.89 & 77.64 $\pm$ 58.8 & 100 & 100 & 5.5 & 6.9 & 8.88 $\pm$ 2 & 9.49 $\pm$ 2.58 \\ \cline{2-11} 
 & \multirow{2}{*}{10} & 5 & 19.13 $\pm$ 21.81 & 66.62 $\pm$ 56.51 & 100 & 100 & 4.21 & 6.9 & 9.45 $\pm$ 1.42 & 9.19 $\pm$ 2.53 \\ \cline{3-11} 
 &  & 9 & 30.56 $\pm$ 24.08 & 60.81 $\pm$ 43.66 & 100 & 100 & 4.3 & 6.14 & 6.87 $\pm$ 1.62 & 8.45 $\pm$ 1.69 \\ \hline
\multirow{3}{*}{3} & 1 & 5 & 83.33 $\pm$ 68.54 & 112.33 $\pm$ 87.24 & 100 & 24 & 8.69 & 9.38 & 11.5 $\pm$ 3.21 & 13.7 $\pm$ 2.74 \\ \cline{2-11} 
 & 5 & 5 & 62 $\pm$ 80.97 & 224.60 $\pm$ 309.67 & 100 & 100 & 8 & 7.3 & 9.88 $\pm$ 1.66 & 10.8 $\pm$ 2.91 \\ \cline{2-11} 
 & 10 & 5 & 18.7 $\pm$ 18.9 & 59.76 $\pm$ 56.73 & 100 & 100 & 5.01 & 6.99 & 8.83 $\pm$ 1.84 & 9.84 $\pm$ 2.15 \\ \hline \hline
\end{tabular}%
}
\resizebox{\textwidth}{!}{%
\begin{tabular}{|c|c|c|c|c|c|c|c|c|c|c|c|c|}
\hline
\multirow{2}{*}{Map} & \multicolumn{3}{c|}{Planning Time (s)} & \multicolumn{3}{c|}{Success Rate (\%)} & \multicolumn{3}{c|}{Solution Cost ($\times 10^5$)} & \multicolumn{3}{c|}{Execution Time (s)} \\ \cline{2-13} 
 & INSAT & Base-A & Base-B & INSAT & Base-A & Base-B & INSAT & Base-A & Base-B & INSAT & Base-A & Base-B \\ \hline
Willow Garage (2.5D) & 40.05 $\pm$ 77.09 & 18.73 $\pm$ 46.7 & 2.55 $\pm$ 1.18 & 100 & 100 & 6 & 3.33 $\pm$ 4.92 & 3.73 $\pm$ 2.54 & 4.34 $\pm$ 1.1 & 14.5 $\pm$ 6.14 & 7.78 $\pm$ 5.33 & 2.54 $\pm$ 2 \\ \hline
Willow Garage (3D) & 57.64 $\pm$ 97.24 & 89.8 $\pm$ 88.31 & 6.53 $\pm$ 2.49 & 100 & 100 & 10 & 5 $\pm$ 5.27 & 1.56 $\pm$ 0.91 & 7.38 $\pm$ 2.67 & 10.11 $\pm$ 6.8 & 3.21 $\pm$ 1.4 & 4.5 $\pm$ 1.78 \\ \hline
MIT Stata Center (3D) & 5 $\pm$ 7.18 & 83.2 $\pm$ 91.94 & 3.94 $\pm$ 1.2 & 100 & 14 & 32 & 6.68 $\pm$ 7.65 & 3.44 $\pm$ 2.33 & 2.24 $\pm$ 0.88 & 7.12 $\pm$ 3.44 & 5.7 $\pm$ 3.34 & 7.7 $\pm$ 2.93 \\ \hline
\end{tabular}%
}
\caption{\small Comparison of INSAT with search-based planning for aggressive SE(3) flight (Base-A) \cite{sikangse3} and polynomial trajectory planning (Base-B) \cite{roypoly}. The top table displays the average and standard deviation of the results for walls and windows environment and the bottom table for indoor office environment. Note that INSAT consistently outperforms the baselines across different types of environments.}
\label{table:struct}
\vspace{-0.7cm}
\end{table*}

\section{Experiments and Results}
\label{sec:results}
\vspace{-0.3cm}

We evaluate the empirical performance of INSAT in simulation against two baselines in two types of environments: 1) a \textit{walls and windows} environment that mimics an array of narrowly spaced buildings each containing several windows smaller than the radius of the quadrotor and 2) a cluttered \textit{indoor office} environment, namely Willow Garage and MIT Stata Center \cite{mitdataset} maps. Together the environments convey a story of a quadrotor aggressively flying through several tall raised office buildings. The baseline methods include search-based planning for aggressive SE(3) flight (Base-A) \cite{sikangse3} and polynomial trajectory planning (Base-B) \cite{roypoly}. We used the AscTec Hummingbird quadrotor \cite{asctec} in the Gazebo simulator \cite{gazebo} as our testing platform. All the methods are implemented in C++ on a 3.6GHz Intel Xeon machine. 
\vspace{-0.3cm}

\begin{figure}
    \begin{subfigure}{\columnwidth}
    \centering
        \includegraphics[width=\columnwidth]{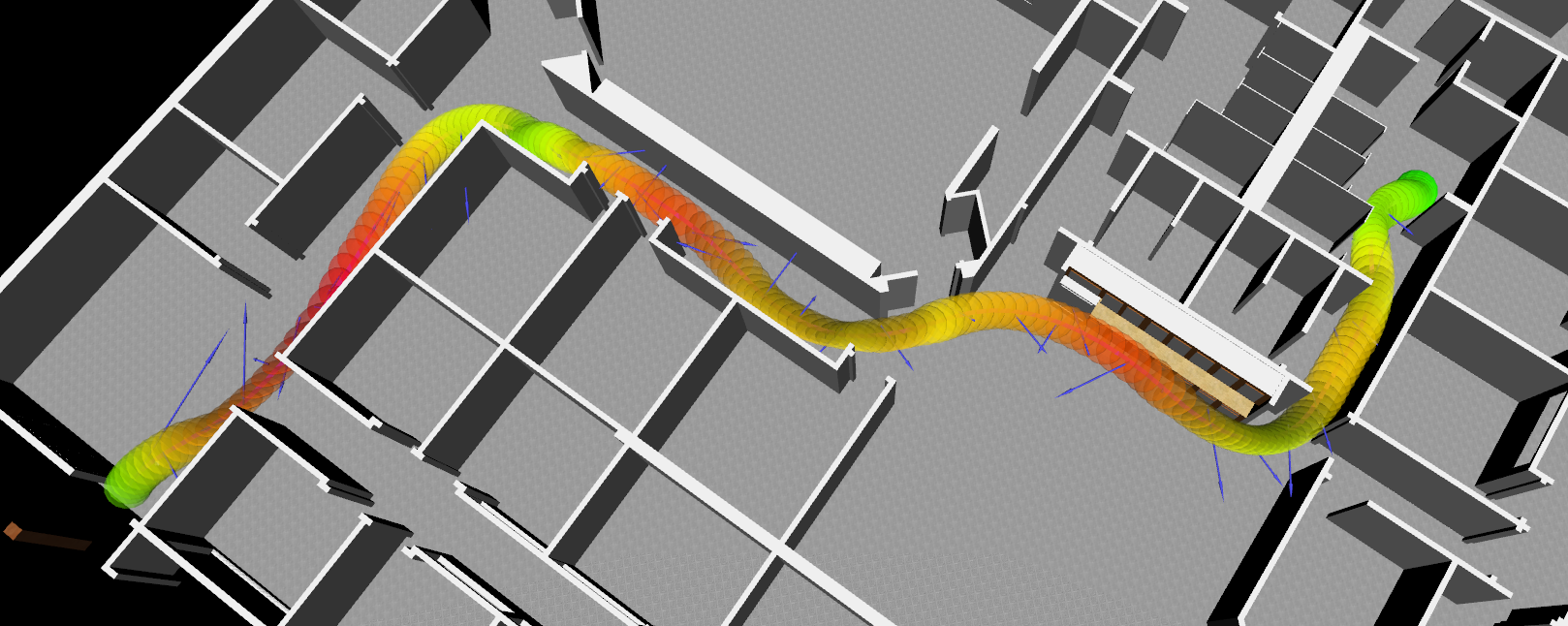}
    \caption{Willow Garage office environment}
    \end{subfigure}
    \begin{subfigure}{\columnwidth}
    \centering
        \includegraphics[width=\columnwidth]{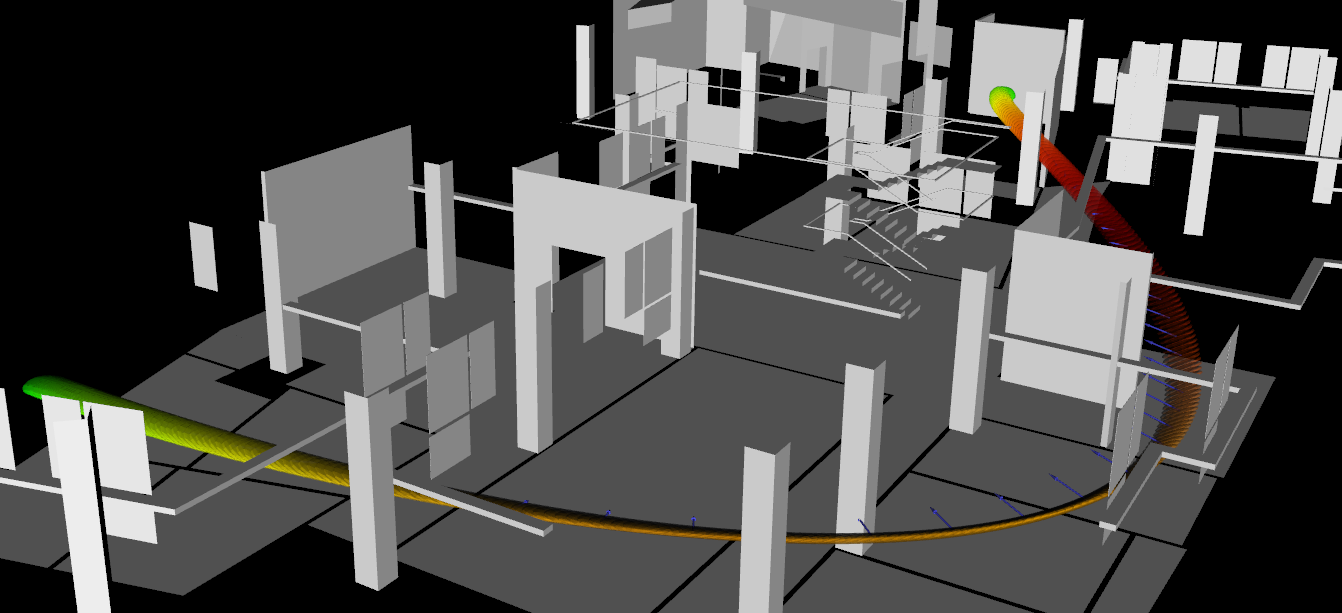}
    \caption{MIT Stata Center}
    \end{subfigure}
    \caption{INSAT in indoor office environments in 3D mode. Trajectories are expressed in the same format as Fig. \ref{fig:walldemo}}
    \label{fig:indoorinsat}
\vspace{-0.1cm}
\end{figure}

\subsection{Walls and Windows Environment}
\label{sec:structenv}
For the walls and windows environment, we randomly generated several scenarios with arbitrary number of parallel walls where each wall contains random number of windows (gaps smaller than quadrotor's radius). The goal of the planner is to generate a trajectory to fly from one end of the parallel walls to the other by negotiating the windows and satisfying their corresponding attitude constraints. Note that the planner also has to figure out the right topology for the solution, \textit{i.e.}, the sequence of windows it can fly through to get to its goal. For this environment, we compared only against Base-A \cite{sikangse3} as the other baseline (Base-B) \cite{roypoly} does not explicitly handle attitude constraints needed to plan in these scenarios and therefore had a very low success rate. 


The planned trajectory from INSAT is visualized (Fig. \ref{fig:walldemo}) as a sequence of ellipsoids approximating the shape of quadrotor to handle SE(3) constraints. We evaluated INSAT and Base-A over 50 random start and goal states in each of the different environment scenarios (top half of Table \ref{table:struct}). For the heuristic, we computed 3D (x, y, z) distances to the goal while accounting for obstacles and assuming a point robot. To compute it, we ran a single 3D Dijkstra's search backwards from the goal to obtain distances for all the cells. The results show that INSAT consistently outperforms Base-A in terms of the trajectory computation time and execution time. All the methods are timed out after 300s. The success rate shows that INSAT finds a solution in every scenario as opposed to Base-A's varying levels of reliability. Specifically, we see that the Base-A struggles when the number of windows per wall is decreased making the planner vary altitude and find a window at different height to get through. This is because Base-A is a lattice search method whose performance strongly depends on parameters such as the density and the length of primitives in the lattice. While reproducing the results in their paper \cite{sikangse3}, we found that their planner used a 2.5D lattice (primitives are restricted to a single plane). Our scenario requires planning in 3D with varying altitude. Despite tuning the parameters to fit 3D configuration for Base-A, the exponential increase in computation combined with the discretization introduced by the lattice sacrificed their success rate.

\subsection{Indoor Office Environment}
We also tested INSAT on the same maps and planning dimensions reported in the baseline papers \textit{i.e} maps of Willow Garage (2.5D \cite{sikangse3} and 3D) and MIT Stata Center (3D) \cite{roypoly}. These are large, cluttered, office environments that contain a number of narrow gaps smaller than the size of quadrotor. The final trajectory from one example is shown in Fig. \ref{fig:indoorinsat} and the statistics are provided in the bottom half of Table. \ref{table:struct}. Willow Garage map has tight spaces and high obstacle density but uniform obstacle distribution along its height compared to the MIT map that has scattered obstacles with varying distribution. Thus, Base-B performs well only in the MIT map as it does not necessitate attitude constrained planning.

From the bottom half of Table \ref{table:struct} we see that INSAT has the highest success rate. For the baselines, we used the same parameters supplied by the authors. In 2.5D planning, Base-A is faster than INSAT as it has a low branching factor with precomputed motion primitives. However, this difference vanishes in 3D because of exponential complexity with longer times spent to escape local minimas in Base-A and relatively faster speeds of polynomial trajectory generation in INSAT. The parameters that determine INSAT's performance including planning time, continuity and obeying dynamic constraints are:
\vspace{-0.4cm}
\begin{table}[h!]
\centering
\resizebox{\columnwidth}{!}{%
\begin{tabular}{|c|c|c|c|c|c|c|c|}
\hline
$d\textbf{x}$ & $d\boldsymbol{\Theta}$ & $\textbf{\.x}_{max}$ & $\textbf{\"x}_{max}$ & $\dddot{\textbf{x}}_{max}$ & $\gamma$ & $f_{max}$ & $dt$ \\ \hline
0.2m & 0.1rad & 10m/s & $20m/s^2$ & $50m/s^3$ & 500 & 10N & 0.05s \\ \hline
\end{tabular}%
}
\end{table}
\vspace{-0.3cm}
\FloatBarrier where $d\textbf{x}$ and d$\boldsymbol{\Theta}$ are the linear and angular discretization used for low-dimensional search, $f_{max}$ is the maximum thrust, $dt$ is the time step used for collision checking and $\gamma$ is the penalty to prioritize control effort over execution time. The execution and trackability of the generated trajectories are evaluated in Gazebo simulator \footnote{A movie of INSAT in Gazebo simulator is available \href{https://bit.ly/2E34C66}{here}.}. One critical parameter is the resolution of the low-dimensional grid that guarantees the planner's completeness (refer Sec. \ref{sec:completeness}). 



\vspace{-0.4cm}
\subsection{INSAT vs Sequential (S) vs Lattice Search (L) methods}
S methods \cite{roypoly} like Base-A first search for a path ignoring the dynamics and then refine to find the feasible trajectory using trajectory optimization. L methods \cite{sikangse3} like Base-B discretize the entire full-dimensional space and precompute the lattice with motion primitives offline. INSAT finds plans with superior behavior compared to S and L because:

\textbf{Computational Complexity:} L methods have fundamental limitation as their performance significantly depends on the choice of discretization for the state and action space, the primitive length along which the control input is constant and the lattice density itself \cite{primcomplexity}. Additionally, solving the boundary value problem to generate primitives that connect the discrete cell centers can be difficult or impossible \cite{primcomplexity}. In our method, albeit $\mathcal{X}_L^{free}$ is discretized, there is no such discretization in $\mathcal{X}_H^{free}$, where we let the optimization figure out the continuous trajectory that minimizes the cost function (Eq. \ref{eq:obj}). As S methods decouple planning in $\mathcal{X}_L$ and $\mathcal{X}_H$, they cannot handle attitude constraints and is restricted to a path found in $\mathcal{X}_L$ when planning in $\mathcal{X}_H$. In S, replacing the entire trajectory found in $\mathcal{X}_L$ with tunnel trajectory (Fig. \ref{fig:tunnel}) can violate the limits of velocity or jerk. Note that INSAT actively tries to minimize such violations (lines \ref{line:wpsearch}-\ref{line:feas2}). Thus, as substantiated by our experiments, interleaving these schemes provide a superior alternative by minimizing the effect of discretization and keeping the full dimensional search tractable.

\textbf{Energy Accumulation Maneuvers:} In tight spaces, a quadrotor might have to perform a periodic swing or revisit a state to accumulate energy and satisfy certain pose constraints. So a high-dimensional trajectory solution might require revisiting a low-dimensional state with a different value for the high-dimensional variables (\emph{i.e.} same $\textbf{x}$ but different $\dot{\textbf{x}}$ or $\dddot{\textbf{x}}$). This is handled by duplicating the low-dimensional state if it is already expanded (lines \ref{line:closed}-\ref{line:deepcopy}). S methods cannot handle this case as they decouple planning in $\mathcal{X}_L$ and $\mathcal{X}_H$. Consequently, observe in Fig. \ref{fig:walldemo} that to negotiate a window in the wall, the quadrotor actively decides to fly in either direction relative to the window to accumulate energy such that an attitude constraint via acceleration (Eq. \ref{eq:constraintmatrix}) can be satisfied at the window. Another interesting behavior is the decision to fly down or rise up helically (Fig. \ref{fig:walldemoa}-CENTER and Fig. \ref{fig:walldemob}-CENTER) in between the tightly spaced walls in order to maintain stability or potentially avoid vortex ring states and simultaneously not reduce the speed by taking slower paths. Such a behavior leveraging the dynamic stability of the quadrotor along with the choice of windows to fly through via global reasoning is a direct consequence of interleaving trajectory optimization with grid-based search. 

\vspace{-0.4cm}
\section{Conclusion}
\label{sec:conc}


We presented INSAT, a meta algorithmic framework that interleaves trajectory optimization with graph search to generate kinodynamically feasible trajectories for aggressive quadrotor flight. We show that interleaving allows a flow of mutual information and help leverage the simplicity and global reasoning benefits of heuristic search over non-convex obstacle spaces, and mitigate the bottleneck introduced by the number of search dimensions and discretization using trajectory optimization. 

The trajectory generation method and graph search algorithm can be easily replaced with alternatives depending on the application. We also analysed the completeness property of the algorithm and demonstrated it on two very different environments. Finally, we note that our method is not just limited to quadrotor planning and can be easily applied to other systems like fixed-wing aircraft or mobile robots that have differentially flat representations \cite{diffflatsys}. To the best of our knowledge, INSAT is the first to interleave graph search with trajectory optimization for robot motion planning.

\vspace{-0.4cm}
\bibliographystyle{IEEEtran}
\bibliography{IEEEabrv,mybibfile}

\end{document}